\documentclass[10pt,twocolumn,letterpaper]{article}

\usepackage{cvpr}
\usepackage{graphicx}
\usepackage{comment}
\usepackage{amsmath,amssymb} % define this before the line numbering.
\usepackage{color}
\usepackage{times}
\usepackage{epsfig}
\usepackage{diagbox}
\usepackage{afterpage}
\usepackage{makecell}
\usepackage{lineno}
\usepackage{amsfonts}       % blackboard math symbols
\usepackage{amsthm}   %\newtheorem{proof}{Proof}

\usepackage{bm}
\usepackage{algorithm}
\usepackage{algorithmic}
\usepackage{amssymb}

\newcommand{\be}{\begin{eqnarray}}
\newcommand{\ee}{\end{eqnarray}}
\newcommand{\beq}{\begin{equation}\begin{aligned}}
\newcommand{\eeq}{\end{aligned}\end{equation}}
\newcommand{\beqn}{\begin{equation*}\begin{aligned}}
\newcommand{\eeqn}{\end{aligned}\end{equation*}}
\newcommand{\ben}{\begin{eqnarray*}}
\newcommand{\een}{\end{eqnarray*}}
\newcommand{\bena}{\begin{eqnarray*}\begin{aligned}}
\newcommand{\eena}{\end{aligned}\end{eqnarray*}}
\newcommand{\bea}{\begin{eqnarray}\begin{aligned}}
\newcommand{\eea}{\end{aligned}\end{eqnarray}}
\usepackage{algorithm}
\usepackage{algorithmic}
\usepackage{multirow}
\usepackage{booktabs}       % professional-quality tables

\usepackage[pagebackref=true,breaklinks=true,letterpaper=true,colorlinks,bookmarks=false]{hyperref}

\usepackage{authblk}
\cvprfinalcopy
\ifcvprfinal\fi

\usepackage{graphicx}
\usepackage{amsmath,amsfonts,amstext,amssymb}

\newtheorem{lemma}{Lemma}

\newtheorem{proposition}{Proposition}

\newtheorem{assumption}{Assumption}
\newcommand{\recheck}[1]{{}}

\begin{document}

\title{How Does GAN-based \\Semi-supervised Learning Work?} % Replace with your title
\author{Xuejiao Liu}
\author{Xueshuang Xiang\thanks{Corresponding author: xiangxueshuang@qxslab.cn}}
\affil{\normalsize Qian Xuesen Laboratory of Space Technology\\ China Academy of Space Technology}
\date{}

%\twocolumn[
%\begin{@twocolumnfalse}
\maketitle
\begin{abstract}
	Generative adversarial networks (GANs) have been widely used and have achieved competitive results in semi-supervised learning. This paper theoretically analyzes how GAN-based semi-supervised learning (GAN-SSL) works. We first prove that, given a fixed generator, optimizing the discriminator of GAN-SSL is equivalent to optimizing that of supervised learning. Thus, the optimal discriminator in GAN-SSL is expected to be perfect on labeled data. Then, if the perfect discriminator can further cause the optimization objective to reach its theoretical maximum, the optimal generator will match the true data distribution. Since it is impossible to reach the theoretical maximum in practice, one cannot expect to obtain a perfect generator for generating data, which is apparently different from the objective of GANs. Furthermore, if the labeled data can traverse all connected subdomains of the data manifold, which is reasonable in semi-supervised classification, we additionally expect the optimal discriminator in GAN-SSL to also be perfect on unlabeled data. In conclusion, the minimax optimization in GAN-SSL will theoretically output a perfect discriminator on both  labeled and unlabeled data by unexpectedly learning an imperfect generator, i.e., GAN-SSL can effectively improve the generalization ability of the discriminator by leveraging unlabeled information.
	%on semi-supervised classification. 
\end{abstract}
%\end{@twocolumnfalse}
%]

\section{Introduction}
Deep neural networks have repeatedly been able to achieve results similar to or beyond those of humans on certain supervised classification tasks based on large numbers of labeled samples. 
In the real world, due to the limited cost of labeling and the lack of expert knowledge, the dataset we obtain usually contains large numbers of unlabeled samples and only a small number of labeled samples. 
Although unlabeled samples do not have label information, they originate from the same data source as the labeled data. 
Semi-supervised learning (SSL) strives to make use of the model assumptions of unlabeled data distributions, thereby greatly reducing the need of a task for labeled data.

Early semi-supervised learning methods include self-training, transductive learning, generative models and other learning methods. 
With the rapid development of deep learning in recent years, semi-supervised learning has gradually been combined with neural networks, and corresponding achievements have occurred \cite{springenberg2016unsupervised,laine2017temporal,kipf2017semi, Berthelot2019mixmatch,xie2019unsupervised}. 
At the same time, the deep generative model has become a powerful framework for modeling complex high-dimensional datasets \cite{Kingma2014Auto,goodfellow2014generative,oord2016wavenet}.
%Deep directed generative model has emerged as a powerful framework for modeling complex high-dimensional datasets.
As one of the earlier works, \cite{kingma2014semi} uses the deep generative model in SSL by maximizing the lower variational bound of the unlabeled data log-likelihood 
and treats the classification label as an additional latent variable in the directed generative model.
%VAE was a new SSL framework based on the generated model was proposed by using a rich parameter density estimator formed by the fusion of probability modeling and deep neural network. 
In the same period, the well-known generative adversarial networks (GANs) \cite{goodfellow2014generative} were proposed as a new framework for estimating generative models via an adversarial process. 
The real distribution of unsupervised data can be learned by using the generator of GANs.
Recently, GANs have been widely used and have obtained competitive results for semi-supervised learning \cite{Salimans2016Improved,springenberg2016unsupervised,Dai2017good,Lecouat2018Manifold,Dong2019marginGAN}.
The early classic GAN-based semi-supervised learning (GAN-SSL) \cite{Salimans2016Improved} presented a variety of new architectural features and training procedures, such as feature matching and minibatch discrimination techniques, to encourage the convergence of GANs. By labeling the generated samples as a new "generated" class, the above techniques were introduced into semi-supervised tasks, and the best results at that time were achieved. 

Despite the great empirical success achieved using GANs to improve semi-supervised classification performance, many theoretical issues are still unresolved. For example, \cite{Salimans2016Improved} observed that, compared with minibatch discrimination, semi-supervised learning with feature matching performs better, but generates images of relatively poorer quality.
%To the best of our knowledge, few investigations have focused on this phenomenon of GAN-SSL. 
A basic theoretical issue of the above phenomenon is as follows: \textbf{how does GAN-based semi-supervised learning work?} 
The theoretical work of \cite{Dai2017good} claimed that good semi-supervised learning requires a "bad" generator, 
which is necessary to overcome the generalization difficulty in the low-density areas of the data manifold. 
Thus, they further proposed a "complement" generator to enhance the effect of the "bad" generator. 
Although the complement generator is empirically helpful for improving the generalization performance, we find that their theoretical analysis was not rigorous; 
%such that the judgment about the requirement on "bad" generator is somehow unreasonable. 
hence, their judgment regarding the requirement of a "bad" generator is somewhat unreasonable.
This paper re-examines this theoretical issue and obtains some theoretical observations inconsistent with that in \cite{Dai2017good}. 

First, we study the relationship between the optimal discriminator in GAN-SSL and the one in the corresponding supervised learning. 
For a fixed generator, we prove that maximizing the GAN-SSL objective is indeed equivalent to maximizing the supervised learning objective. 
That is, the optimal discriminator in GAN-SSL is expected to be perfect in the corresponding supervised learning, 
i.e., the optimal discriminator can make a correct decision on all labeled data. 
Thus, we can initially state that \textbf{GAN-SSL can at least obtain a good performance on labeled data}. 

Second, for the optimal discriminator, we investigate the behavior of the optimal generator of GAN-SSL. We find that the error of the optimal generator distribution relative to the true data distribution highly depends on the optimal discriminator. Given an optimal discriminator, the minimax game in GAN-SSL for the generator is equivalent to minimizing $-KL(p||p-\epsilon p_{G})+2JS(p||p_{G})$, where $p$ is the true data distribution, $p_{G}$ is the generator distribution and $\epsilon$ represents the error between the output of the perfect discriminator and its theoretical maximum.
Once the optimal discriminator can not only make a correct decision on all labeled data 
but also cause the objective of the supervised learning part in GAN-SSL to reach its theoretical maximum (i.e., $\epsilon=0$), 
minimizing the generator objective is equivalent to minimizing $JS(p||p_{G})$; hence, the optimal generator is perfect, i.e., the optimal generator distribution is indeed the true data distribution. 
Whereas the theoretical maximum is impossible to reach in practice, the optimal generator will always be inconsistent with the true data distribution, 
i.e., the optimal generator is imperfect. 
Notably, although our observation is that the optimal generator is imperfect in practice, it is not as "bad" as claimed in \cite{Dai2017good}. 
Here, we can state that \textbf{GAN-SSL will always output an imperfect generator.} 

Furthermore, if the labeled data can traverse all connected subdomains of the data manifold, which is a reasonable assumption in semi-supervised classification, we will additionally have that the optimal discriminator of GAN-SSL can be perfect on all unlabeled data, i.e., it can make a correct decision on all unlabeled data. Thus, we can state that \textbf{GAN-SSL can also obtain a good performance on unlabeled data}. 

Overall, with our theoretical analysis, we can answer the above question: the optimal discriminator in GAN-SSL can be expected to be perfect on both labeled 
and unlabeled data by learning an imperfect generator. 
This theoretical result means that GAN-SSL can effectively improve the generalization ability of the discriminator by leveraging unlabeled information.
\section{Related Work}
Based on the classic GAN-SSL \cite{Salimans2016Improved}, 
%lots of interests have been put on the design of additional neural networks and its corresponding objectives. 
the design of additional neural networks and the corresponding objectives have received much interest. 
The work of \cite{Li2017Triple} presents Triple-GAN, which consists of three players, i.e., a generator, a discriminator and a classifier, 
to address the problem in which the generator and discriminator (i.e., the classifier) in classic GAN-SSL may not be optimal at the same time. 
Later, a more complex architecture consisting of two generators and two discriminators was developed \cite{Gan2017Triangle}, which performs better than Triple-GAN under certain conditions.
To address the issue caused by incorrect pseudo labels, based on the margin theory of classifiers, MarginGAN \cite{Dong2019marginGAN}, which also adopts a three-player architecture, was proposed.

As mentioned in GAN-SSL \cite{Salimans2016Improved}, feature matching works very well for GAN-SSL, but the interaction between the discriminator $D$ and the generator $G$ is not yet understood, especially from the theoretical aspect. 
To the best of our knowledge, few theoretical investigations on this topic exist, with one exception being the work of \cite{Dai2017good}, which focused on the analysis of the optimal $G$ and $D$. 
Our paper is also developed from this aspect. 
In addition, some theoretical works in the area of GANs regarding the design of the objective~\cite{salimans2018improving,Mescheder2018Which} 
or the convergence of the adversarial process~\cite{Arora2017Generalization,Bai2019Approximability}, can be applied 
to the area of GAN-SSL. 
Since GAN-SSL has an additional supervised learning objective, 
one should further develop some theoretical techniques to handle this problem. 
These open aspects should be studied in the future.

\section{Theoretical Analysis}
We first introduce some classic notations and definitions. 
We refer to~\cite{Salimans2016Improved} for a detailed discussion. 
Consider a standard classifier for classifying a data point $x$ into one of $K$ possible classes. 
The output of the classifier is a $K$-dimensional vector of logits that can be turned into class probabilities by applying softmax.  
The classic GAN-SSL~\cite{Salimans2016Improved} is implemented by labeling samples from the GAN generator $G$ with a new "generated" class $K + 1$. 
We use $P_{D}(y\leq K|x), P_{D}(K+1|x)$ to determine the probability that $x$ is true or fake, respectively. 
The model $D$ is both a discriminator and a classifier, correspondingly increasing the dimension of the output from $K$ to $K + 1$.
Thus, the discriminator $D$ is defined as $P_{D}(k|x) = \frac{\exp(\omega_{k}^{T}f(x))}{\sum_{i=1}^{K+1}\exp(\omega_{i}^{T}f(x))}$, 
where $f(x)$ is a nonlinear vector-valued function, $\omega_{k}$ is the weight vector for class $k$ 
and $\{\omega_{1}^{T}f(x), \ldots, \omega_{K+1}^{T}f(x)\}$ is the $K+1$-dimensional vector of the logits.
Since a discriminator with $K+1$ outputs is over-parameterized, $\omega_{K+1}$ is fixed as a zero vector. We also denote $D:=(\omega,f)$ as a discriminator. 
Similar to the traditional GANs, in GAN-SSL, $D$ and $G$ play the following two-player minimax game with the value function $J_{GD}$:
%\bea\label{eq-minimax}
%\min_G\max_{D}J(G,D) &= \mathbb{E}_{(x,y)\sim p(x,y)}\log P_{D}(y|x, y\leq K)\\
%&+\mathbb{E}_{x\sim p(x)}\log P_{D}(y\leq K|x) 
%+\mathbb{E}_{x\sim p_{G}(x)}\log P_{D}(K+1|x)
%\eea 
\bea\label{eq-minimax}
\min_G\max_{D}J_{GD} = \min_G\max_{D}(L_{D} + U_{GD})
\eea 
where  
\bena
L_{D} &= \mathbb{E}_{(x,y)\sim p(x,y)}\log P_{D}(y|x, y\leq K)\\
U_{GD} &= \mathbb{E}_{x\sim p(x)}\log P_{D}(y\leq K|x) \\
&+\mathbb{E}_{x\sim p_{G}(x)}\log P_{D}(K+1|x)
\eena 
where $L_D$ is the supervised learning objective for all labeled data and $U_{GD}$ is the unsupervised objective, i.e., the traditional GANs objective, 
in which $p$ is the true data distribution and $p_{G}$ is the generator distribution. 
When $G$ is fixed, the objectives $J_{GD}$ and $U_{GD}$ become $J_{D}$ and $U_{D}$, respectively. 
After we obtain a satisfactory $D$ by optimizing $J_{GD}$, we use ${\rm argmax}\,P_{D}(k|x, k\leq K)$ to determine the class of the input data $x$. 
Similar to the other related works on GAN-SSL, we use $L_{D}$ as the supervised learning objective by applying only the softmax operator on the former $K$-dimensional vector of the output of $D$.

Similar to the theoretical analysis of GANs, we consider a nonparametric setting, e.g., we represent a model with infinite capacity and study the optimal discriminator $D$ and generator $G$ in the space of probability density functions. The theoretical proofs, including Proposition~\ref{prop1} - Proposition~\ref{prop3}, are provided in the supplementary material. 
We show the proof of Lemma \ref{lem1}, since it is the key point of our paper. 

\subsection{Optimal discriminator on labeled data}
First, we consider the optimal discriminator $D$ for any given generator $G$. 
Motivated by the proof of Proposition 1 in \cite{Dai2017good}, this subsection proves that for fixed $G$, if the discriminator has infinite capacity, then regardless of whether the generator is perfect, i.e., $p_G = p$ or $p_G \neq p$, 
maximizing the above GAN-SSL objective $J_{D}$ is equivalent to maximizing the supervised learning objective $L_{D}$. We first introduce a basic Lemma \ref{lem1} and give its proof. 
\begin{lemma}\label{lem1}
	For any given generator $G$, if the discriminator has infinite capacity, then for any solution $D = (\omega,f)$ of the GAN-SSL objective $J_D$, there exists another solution $D^{*} = (\omega^{*}, f^{*})$ such that $ L_{D^{*}} = L_{D}$ and $U_{D^{*}} \geq U_{D}$.
\end{lemma}
%\emph{Proof}. 
\begin{proof}
For any given generator $G$ and any solution $D = (\omega, f)$ of the GAN-SSL objective $J_D$, because the discriminator has infinite capacity, 
there exists $D^{*} = (\omega^{*}, f^{*})$ such that for all $x$ and $k\leq K$,
\begin{equation*}
\exp(\omega_{k}^{*T}f^{*}(x)) = \frac{\exp(\omega_{k}^{T}f(x))}{\sum_{i=1}^{K}\exp(\omega_{i}^{T}f(x))}\cdot \frac{p(x)}{p_{G}(x)}
\end{equation*}
For all $x$,
\begin{align*}
&P_{D^{*}}(y|x, y\leq K) 
= \frac{\exp(\omega_{y}^{*T}f^{*}(x))}{\sum_{i=1}^{K}\exp(\omega_{i}^{*T}f^{*}(x))}\\
=& \frac{\exp(\omega_{y}^{T}f(x))}{\sum_{i=1}^{K}\exp(\omega_{i}^{T}f(x))} 
= P_{D}(y|x, y\leq K)
\end{align*}
Then $L_{D^{*}} = L_{D}$, and
\begin{align*}
&P_{D^{*}}(K+1|x) 
= \frac{\exp(\omega_{K+1}^{*T}f^{*}(x))}{\sum_{i=1}^{K+1}\exp(\omega_{i}^{*T}f^{*}(x))}\\
=& \frac{1}{1+\sum_{i=1}^{K}\exp(\omega_{i}^{*T}f^{*}(x))} 
= \frac{p_{G}(x)}{p(x)+p_{G}(x)}	
\end{align*}
For the following unsupervised objective function
\begin{align*}
U_{D} &= \mathbb{E}_{x\sim p(x)}\log (1-P_{D}(K+1|x)) \\
&+ \mathbb{E}_{x\sim p_{G}(x)}\log P_{D}(K+1|x) \\
&=\int_{x}p(x)\log(1-P_{D}(K+1|x)) \\
&+ p_{G}(x) \log P_{D}(K+1|x)dx
\end{align*}
when $G$ is fixed, it is easy to verify that $D^{*}$ can maximize it. 
%\begin{equation*}
%P_{D}(K+1|x)= \frac{p_{G}(x)}{p(x)+p_{G}(x)}.
%\end{equation*}
Therefore, $L_{D^{*}} = L_{D}$ and $U_{D^{*}} \geq U_{D}$. 
This completes the proof.
\end{proof}

Lemma \ref{lem1} states that for any given generator and due to the infinite capacity of the discriminator, under the condition that the supervised objective remains unchanged, we can always increase the unsupervised objective until the extreme value is reached. 
Then, based on Lemma \ref{lem1}, we can present the theoretical results of the optimal discriminator $D$ for any given generator $G$. 
\begin{proposition}\label{prop1}
	Given the conditions in Lemma \ref{lem1}, we can obtain the following:\\
	(1) for any optimal solution $D_{L} = (\omega,f)$ of the supervised learning objective $L_{D}$, there exists $D^{*} = (\omega^{*}, f^{*})$ such that $D^{*}$ maximizes the GAN-SSL objective $J_{D}$ and that, for all $x$, 
	\begin{equation*}
	P_{D^{*}}(y|x, y\leq K) = P_{D_{L}}(y|x, y\leq K).
	\end{equation*}
	(2) 
	%optimizing the discriminator of GAN-SSL is equivalent to optimizing the one of supervised learning. 
	for any optimal solution $D^{*} = (\omega^{*}, f^{*})$ of the above GAN-SSL objective $J_{D}$, $D^{*}$ is an optimal solution of the supervised objective $L_{D}$.\\
	(3) the optimal discriminator $D^{*}$ of the above GAN-SSL objective is 
	\begin{align*}
	P_{D^{*}}(y\leq K|x) &= \frac{p(x)}{p(x)+p_{G}(x)}, \\
	P_{D^{*}}(K+1|x) &= \frac{p_{G}(x)}{p(x)+p_{G}(x)}.
	\end{align*}
\end{proposition}
Proposition \ref{prop1} (1) and (2) jointly indicate that given a fixed generator, optimizing the discriminator of GAN-SSL is equivalent to optimizing that of supervised learning. 
Thus, the optimal discriminator in GAN-SSL is expected to be perfect on labeled data. 
We should note that Proposition $1$ of~\cite{Dai2017good} gives a similar theoretical result, claiming only that 
the optimal discriminator of $L_D$ can lead to an optimal discriminator of GAN-SSL, i.e., the first 
claim of the above Proposition \ref{prop1}, but given the condition that $p_G = p$. 
Actually, our theoretical investigation instead shows that the condition $p_G = p$ is unnecessary and that the optimal discriminator of GAN-SSL 
is an optimal discriminator of supervised learning. 

By carefully checking the statements in Proposition $1$ of~\cite{Dai2017good} 
and comparing their results with ours, one may realize that their results are not rigorous. 
They claimed that good semi-supervised learning requires a bad generator because they observed that 
the semi-supervised objective shares the same generalization error with the supervised objective given a perfect generator. 
First, from our theoretical results, a perfect generator is not necessary for the observation, such that the requirement of a bad generator in their claim is unreasonable. 
Moreover, the fact that the semi-supervised objective shares the same generalization error with the supervised objective is insufficient for claiming 
that GAN-SSL will reduce the generalization ability of supervised learning. 
In addition, if we can claim only that the optimal discriminator of $L_D$ can lead to the optimal discriminator of GAN-SSL, 
one would not expect to obtain a perfect discriminator on labeled data by GAN-SSL 
such that the Assumption $1$ (1) of~\cite{Dai2017good} is over-assumed. 
Our result, i.e., the Proposition \ref{prop1} (2), can overcome their shortcoming, 
such that our later Assumption \ref{assump1} (1) is indeed reasonable.

\subsection{Optimal generator}
Next, for a fixed optimal discriminator, we discuss the behavior of the optimal generator of GAN-SSL.
By the definition of the discriminator $D$, we have $P_{D}(K+1|x) ={1}/{(1+\sum_{i=1}^{K}\exp(\omega_{i}^{T}f(x)))}$. 
Based on Proposition \ref{prop1} (3), $P_{D}(K+1|x) ={p_{G}(x)}/{(p(x)+p_{G}(x))}$. 
Namely,
\begin{align*}
\frac{1}{1+\sum_{i=1}^{K}\exp(\omega_{i}^{T}f(x))} = \frac{p_{G}(x)}{p(x)+p_{G}(x)}
\end{align*}
implying that $\sum_{i=1}^{K}\exp(\omega_{i}^{T}f(x)) = p(x)/p_{G}(x)$. 
Now, the objective of $G$ becomes $C(G) = L_{DG} + U_{G}$, with
\begin{align*}
L_{DG} &= \mathbb{E}_{(x,y)\sim p(x,y)}\log P_{D}(y|x, y\leq K)\\
&= \mathbb{E}_{(x,y)\sim p(x,y)}\log(\exp(\omega_{y}^{T}f(x)) \frac{p_{G}(x)}{p(x)})\\
U_{G} &= \mathbb{E}_{x\sim p}\log P_{D}(y\leq K|x) 
+\mathbb{E}_{x\sim p_{G}}\log P_{D}(K+1|x)\\
&= \mathbb{E}_{x\sim p}\log\frac{p(x)}{p(x)+p_{G}(x)} + \mathbb{E}_{x\sim p_{G}}\log \frac{p_{G}(x)}{p(x)+p_{G}(x)}
\end{align*}

%Now, we show the global optimization theory of the generator for this semi-supervised learning objective. This work 
The aim of the discriminator is to classify the labeled samples into the correct class, the unlabeled samples into the "true" class (any of the first $K$ classes) 
and the generated samples into the "fake" class, i.e., the $(K+1)$-th class. 
%As analyzed above, the optimized discriminator in GAN-SSL is expected to be a perfect one on the labeled data, then if 
%so-called perfect discriminator means that the discriminator can reach the theoretical maximum, i.e.,
In this paper, the so-called \textbf{perfect discriminator} means that for any $(x,y)\in p(x,y)$, $P_{D}(y|x, y \leq K) > P_{D}(k|x, k \leq K, k \neq y)$, i.e., the discriminator can make a correct decision on labeled data. 
We say that the discriminator reaches its \textbf{theoretical maximum} given that $P_{D}(y|x, y \leq K) =1$ and, for any other class $k\neq y$, $P_{D}(k|x, k \leq K) =0$. 
Obviously, it is impossible to reach the theoretical maximum in practice, i.e., $P_{D}(y|x, y \leq K)$ can only be close to $1$, and for any other class $k\neq y$, $\exp (\omega_{k}^{T}f(x))\rightarrow 0$. That is, because the discriminator has an infinite capacity, we can expect that there exists $0<\epsilon \ll 1$ such that $(K-1)\exp (\omega_{k}^{T}f(x))\leq \epsilon$ for any other class $k\neq y$. Here, $\epsilon$ represents the error between the output of the perfect discriminator and its theoretical maximum. 
Now, we show the relationship between the true data distribution and the generator distribution in GAN-based semi-supervised tasks through the following Proposition \ref{prop2}. 
%Then we can prove the following Proposition \ref{prop2}.
\begin{proposition}\label{prop2}
	Given the conditions in Proposition \ref{prop1}, for a fixed optimal discriminator $D$, suppose there exists $0<\epsilon\ll 1$ such that the other logit output ($k\neq y$) of $D$ satisfies $(K-1)\exp (\omega_{k}^{T}f(x))\leq \epsilon$. Then, the generator objective $C(G)=-KL(p||p-\epsilon p_{G})+2JS(p||p_{G})-2\log 2$.
	%$JS(p||p_{G})$ and maximize $KL(p||p-\epsilon p_{G})$ simultaneously.
	%achieved if and only if $p_{G} = p$ and if $0<\epsilon<1$, then $p_{G}\neq p$.
\end{proposition}
Proposition \ref{prop2} indicates that minimizing the generator objective $C(G)$ is equivalent to minimizing $JS(p||p_{G})$ and maximizing $KL(p||p-\epsilon p_{G})$ simultaneously. To minimize $JS(p||p_{G})$, we need to train $p_{G}$ close to $p$; however, to maximize $KL(p||p-\epsilon p_{G})$, we need to increase the difference between $p$ and $p-\epsilon p_{G}$. 
If the optimal solution $D$ of GAN-SSL for the supervised objective can reach the theoretical maximum, i.e., $\epsilon = 0$, then minimizing $C(G)$ is equivalent to minimizing $JS(p||p_{G})$. 
Since the Jensen-Shannon divergence between two distributions is always non-negative and zero iff they are equal, then $C^{*} = -2\log 2$ is the global minimum of $C(G)$, and the only solution is $p_{G} = p$, i.e., the generator distribution perfectly replicates the true data distribution such that the optimal generator is perfect. 
Whereas, it is impossible to reach the theoretical maximum, then $\epsilon \neq 0$.
%However, 
If we assume that $\epsilon$ is sufficiently small, i.e., $0<\epsilon\ll 1$, 
the optimal generator $p_G$ will not be $p$. 
Otherwise, if $p_{G} = p$, $JS(p||p_{G})$ can reach its minimum, and $KL(p||p-\epsilon p_{G})=KL(p||p(1-\epsilon))$ becomes very small, 
which is in contrast to maximizing $KL(p||p-\epsilon p_{G})$ in the objective $C(G)$. 

%then when $p_{G} = p$, $p$ is close to $p-\epsilon p_{G}$, . At this time, $JS(p||p_{G})$ reaches the minimum value, but the value of $KL(p||p-\epsilon p_{G})$ will become smaller.
%This is contrary to the optimization goal of $KL(p||p-\epsilon p_{G})$. 
%Therefore, minimizing $JS(p||p_{G})$ and maximizing $KL(p||p-\epsilon p_{G})$ play the two-player minimax game. 
%To learn a better discriminator for the classification task, the $p_{G}$ should not be equal to $p$. 	

Different from the theoretical result of \cite{Dai2017good} that good semi-supervised learning requires a bad generator, 
%we show that the error of optimal generator distribution between the true data distribution highly depends on the optimal discriminator.
Proposition \ref{prop2} shows that the error of the optimal generator distribution relative to the true data distribution largely depends on the optimal discriminator.
First, when the perfect discriminator for the supervised objective can reach its theoretical maximum,  
the global optimal generator of GAN-SSL is $p_{G} = p$, i.e., the optimal generator $G$ is perfect.  
%In fact, the discriminator $D$ cannot be perfect, at this time, minimizing $C(G)$ is equivalent to minimize $JS(p||p_{G})$ and maximize $KL(p||p-\epsilon p_{G})$ simultaneously. To minimize $JS(p||p_{G})$ need to train $p_{G}$ close to $p$, however, for $0<\epsilon \ll 1$, this is contrary to maximize $KL(p||p-\epsilon p_{G})$. \
However, since the theoretical maximum cannot be reached in practice, 
% generator learning data distribution in GAN-SSL is a process of playing the two-player minimax game,
%minimizing $JS(p||p_{G})$ and maximizing $KL(p||p-\epsilon p_{G})$ simultaneously, 
one cannot expect to learn a perfect generator, i.e., the optimal generator $G$ is imperfect. 
In addition, we note that although the optimal generator is imperfect in actual situations, it can be similar to the original data distribution. 
As shown in the work of \cite{Salimans2016Improved}, samples generated by the excellent GAN-SSL generator are not perfect 
but not completely different from the original data distribution, which is apparently inconsistent with the requirement of a "bad" generator in~\cite{Dai2017good}.  

%On this basis, we proved that without the strong condition of the complement generator, under reasonable assumptions of discriminator convergence conditions and dataset manifold connectivity conditions, the discriminator can learn the correct decision boundary for all unlabeled data.

\textbf{Case Study on Synthetic Data.}
To obtain a more intuitive understanding of Proposition \ref{prop2}, we conduct a case study based on a 1D synthetic dataset, where we can easily verify our theoretical analysis by visualizing the model behaviors. The training dataset is sampled from the Gaussian distribution $p = \mathcal{N}(0, 0.4^2)$. Based on Proposition \ref{prop2}, we study the relationship between the original data distribution $p$ and the generator distribution $p_G$ obtained by directly training the following generator objective function $C(G) = -KL(p||p-\epsilon p_{G})+2JS(p||p_{G})$.
We use the generator architecture of a fully connected neural network with ReLU activations: 1-100-100-1.
We give the results of comparing $p$ and $p_G$ when $\epsilon = 0$, $\epsilon = 0.1$ and $\epsilon = 0.2$ in Figure \ref{fig-1}, respectively. 
As theoretically analyzed above, for a constant $\epsilon = 0$, minimizing $C(G)$ is equivalent to minimizing $JS(p||p_{G})$, and the optimal solution is $p_{G} = p$ (as shown in Figure \ref{fig-1} (a)). As shown in Figure \ref{fig-1} (b) and (c), when $\epsilon \neq 0$, $p_{G}$ is imperfect and the greater $\epsilon$ is, the larger the gap between $p_{G}$ and $p$. 
Although there is a certain gap between the generator distribution and the data distribution, $p_{G}$ is similar to $p$, and most of their supports overlap. 
%This is why the data generated by the generator in GAN-SSL is similar to the original data, but it is not exactly the same. 
Therefore, in this case, different from GANs, the optimal generator cannot be used to generate samples.
\begin{figure*}[htbp]
	\centering
	%\begin{center}
	\includegraphics[scale=0.50]{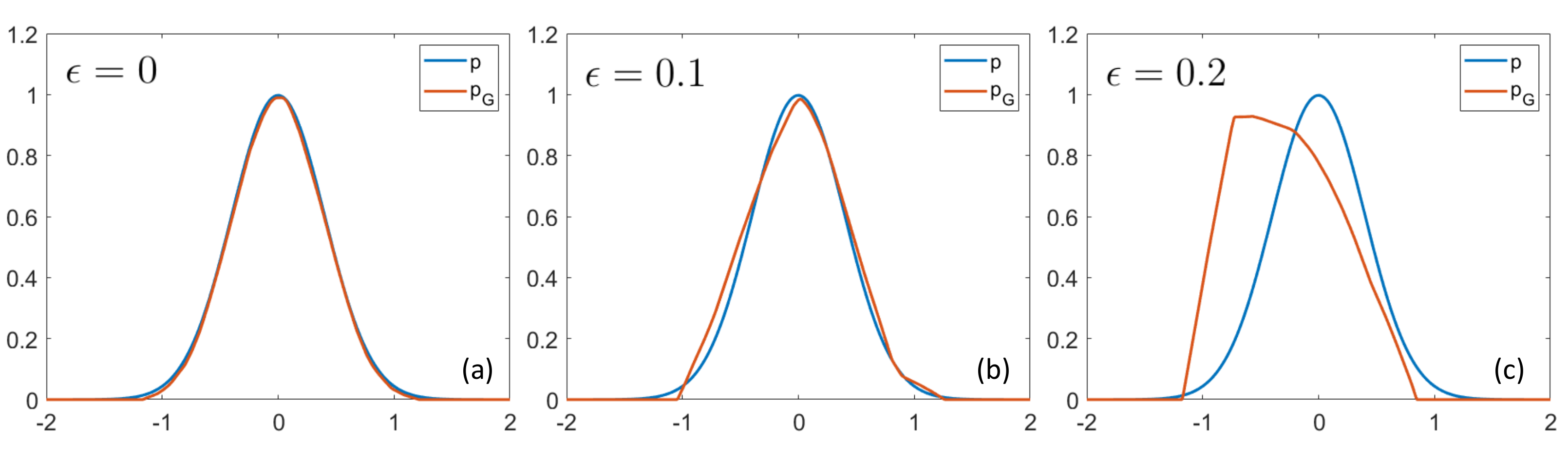}
	\vspace{0.1cm}
	%\end{center}
	\caption{The comparison results of the data distribution $p$ (blue) and the generator distribution $p_G$ (red) when $\epsilon = 0$, $\epsilon = 0.1$ and $\epsilon = 0.2$, respectively. 
		Here, $\epsilon$ can somehow represent the error between the output of the perfect discriminator and its theoretical maximum.}
	\label{fig-1}
\end{figure*}

\subsection{Optimal discriminator on unlabeled data}
Finally, to study the behavior of the optimal discriminator of GAN-SSL on unlabeled data, 
%we will study how GAN-SSL can effectively improve the generalization ability of the discriminator by leveraging the unlabeled information.
we first present the assumption on the convergence conditions of the optimal discriminator and generator and the manifold conditions on the dataset. 
Let $\Omega_{\mathcal{L}}$ be the set of labeled data. 
Let $\Omega$, $\Omega_{\mathcal{G}}$ be the data manifold and generated data manifold, respectively. 
Obviously, we have $\Omega_{\mathcal{L}} \subset \Omega$, and the unlabeled data should be sampled from $\Omega \setminus \Omega_{\mathcal{L}}$. 
Denote $\Omega = \cup_{k=1}^{K}\Omega^{k}$, where $\Omega^{k}$ is the data manifold of class $k$. 
Denote $\Omega^{k} = \cup_l \Omega^{kl}$, 
where $\Omega^{kl}$ is a connected subdomain. 
% We always assume that each $\Omega_{\mathcal{U}}^{k}$ contains at least one labeled data point. 
%The function $F$ maps data points in the input space to the feature space: $F:\{\mathcal{L}, \mathcal{U}\}-> \{F_\mathcal{L}, F_\mathcal{U}\}$, and $F_\mathcal{U} = \cup_{k=1}^{K}F_\mathcal{U}^{k}$, where $F_\mathcal{U}^{k}$ is a subset of the unlabeled data points of class $k$ in the feature space. 

\begin{assumption}\label{assump1}
	%When $D$ converges on a finite training set \{$\mathcal{L}, \mathcal{U}$\}, $D$ learns a (strongly) correct decision boundary for all training data points. 
	Suppose $D, G$ are optimal solutions of the GAN-SSL objective \eqref{eq-minimax}. We assume that:\\
	(1) for any $(x,y)\in \Omega_{\mathcal{L}}$, we have $\omega_{y}^{T}f(x) > \omega_{k}^{T}f(x)$ for any other class $k\neq y$;\\
	(2) for any $x \in \Omega$, $\max_{k=1}^{K}\omega_{k}^{T}f(x) > 0$; for any $x\in \Omega_{\mathcal{G}} \setminus \bar{\Omega}$, $\max_{k=1}^{K}\omega_{k}^{T}f(x) < 0$;\\
	(3) each $\Omega^{jl}$ contains at least one labeled data point;\\
	(4) for any $x_k \in \Omega^{k}$ and $j\neq k$, there exists a connected subdomain $\Omega^{jl}$, a labeled data point $x_j\in \Omega^{jl}$, 
	a generated data point $x_g \in \Omega_{\mathcal{G}} \setminus \bar{\Omega}$, and $0 < \alpha < 1$, such that $f(x_g) = \alpha f(x_k) + (1 - \alpha) f(x_j)$.
	%\end{enumerate}
\end{assumption}
%\textbf{Hypothetical rationality.} 
\textbf{The Reasonableness of Assumption \ref{assump1} (1).} According to the theoretical analysis for the optimal discriminator (Proposition \ref{prop1}), we found that optimizing the discriminator of GAN-SSL is equivalent to optimizing that of supervised learning. Because $D$ is an optimal solution of the GAN-SSL objective, $D$ is an optimal solution of the supervised learning objective. Naturally, Assumption \ref{assump1} (1) holds, i.e., the discriminator has a correct decision boundary for labeled data given that the discriminator has infinite capacity. 
%We should note that as discussed above, the Assumption $1$ (1) of ~\cite{Dai2017good} is over assumed. 

\textbf{The Reasonableness of Assumption \ref{assump1} (2).} 
Proposition \ref{prop2} implies that the optimal generator we obtain in practice is not a perfect generator such that we assume $\Omega_{\mathcal{G}} \setminus \bar{\Omega} \neq \emptyset$. We ignore the case of $\Omega_{\mathcal{G}} \subset \Omega, \Omega \setminus \bar{\Omega}_{\mathcal{G}}\neq \emptyset$ since it is almost impossible to ensure that our optimal generator can output only true data but with different probabilities in practice, unless we obtain the worst model with model collapse. 
Similar to \cite{Dai2017good}, we also make a strong assumption regarding the true-fake correctness of the true data ($\Omega$) and fake data ($\Omega_{\mathcal{G}} \setminus \bar{\Omega}$). 
In other words, we assume that the sampling of unlabeled data is good enough to achieve the best generalization ability on the true data manifold. 
% We should notice that the assumption in \cite{Dai2017good} actually implies $\Omega \cap \Omega_{\mathcal{G}} = \emptyset$, given they claim that the optimal generator should be a complement one. 

\textbf{The Reasonableness of Assumption \ref{assump1} (3).} Intuitively, one cannot expect to achieve a good classification performance on a connected subdomain with no label information, since we can optionally set this connected subdomain to any class but with no influence on the objective. 

\textbf{The Reasonableness of Assumption \ref{assump1} (4).} 
Apparently, Assumption \ref{assump1} (3) is a sufficient condition for this assumption. 
In addition, since the optimal generator is imperfect and we need only its existence, 
one can expect to be able to achieve this condition.  
% We note that the paper \cite{Dai2017good} make a similar but more strong assumption before the definition of their Assumption by set $G$ a complement generator. 

%Assumption\ref{assump1} (1) and (2) are consistent with the goal of the discriminator and can be directly induced by the discriminator objective function. 
%Later, reasonableness of Assumption \ref{assump1} (4) will be verified by case study on synthetic data. 
Now, based on Assumption \ref{assump1}, we give the main result of this subsection.  
%that GAN-SSL indeed effectively improve the generalization ability of the discriminator by leveraging the unlabeled information.
%\begin{proposition}\label{prop3}
%	Given the conditions in Assumption \ref{assump1}, for all class $k \leq K$, for all feature space points $f_{k}\in F_\mathcal{U}^{k}$, there exists at least one label data point $(f_\mathcal{L}^{k},k) \in F_\mathcal{U}^{k}$, then, we have $\omega_{k}^{T}f_{k} > \omega_{j}^{T}f_{k}$ for any $j\neq k$.
%\end{proposition}
\begin{proposition}\label{prop3}
	Given the conditions in Assumption \ref{assump1}, for all classes $k \leq K$, for all data space points $x_k\in \Omega^{k}$, we have $\omega_{k}^{T}f(x_k) > \omega_{j}^{T}f(x_k)$ for any $j\neq k$.
\end{proposition}

We remark that the proof of Proposition \ref{prop3} is similar to that in \cite{Dai2017good} but with different conditions;  
see the supplementary material for more details. 
Proposition \ref{prop3} guarantees that given the convergence conditions of the optimal discriminator and generator, which are induced by 
Proposition \ref{prop1} and \ref{prop2}, respectively, if the labeled data can traverse all subdomains of the data manifold, 
%under reasonable assumptions on the convergence conditions of the optimal discriminator and the manifold conditions on the dataset, 
then the discriminator $D$ will be perfect on the data manifold, i.e., 
it can also learn correct decision boundaries for all unlabeled data. 
%Although a same theoretical result was also given in the Proposition $2$ of \cite{Dai2017good}, our assumptions are indeed different. 
Now, we discuss the difference between our Assumption \ref{assump1} and Assumption $1$ in \cite{Dai2017good}. 
\begin{enumerate}
	\item As discussed above in subsection $3.1$, the first Assumption $1$ (1) in \cite{Dai2017good} is over-assumed, since they cannot expect to obtain a perfect discriminator by optimizing the GAN-SSL objective. However, our Proposition \ref{prop1} can support our Assumption \ref{assump1} (1). 
	\item Their Assumption $1$ (2) and Assumption $1$ (3) correspond to our Assumption \ref{assump1} (2). We all assume that the optimal $D$ can result in a perfect decision regarding true-fake correctness. 
	However, since our optimal $G$ is not a bad one, as a result of Proposition \ref{prop2}, we assume only that the data in $\Omega_{\mathcal{G}} \setminus \bar{\Omega}$ are fake. 
	\item Unlike \cite{Dai2017good}, we additionally give a reasonable assumption on the data manifold, i.e, Assumption \ref{assump1} (3). By a careful check, one can find that the authors of \cite{Dai2017good} ignore this assumption because they implicitly assume that after embedding $f$, the feature space of each class is a connected domain. Otherwise, the proof of their Proposition $2$ is not correct, since one can set a connected subdomain of the feature space $F_k$ without label information to any class by a similar contradiction proof. 
	\item Similar to our Assumption \ref{assump1} (4), \cite{Dai2017good} make a stronger assumption before the definition of their Assumption $1$ by setting $G$ as a complement generator. Given a complement generator, our Assumption \ref{assump1} (4) can clearly be satisfied. 
	Note that under their definition, if the optimal generator $G$ is a complementary one, the feature space of the generated data will fill the complement set of the feature space of the true data. In other words, they hope that the "bad" generator can not only generate fake data but also generate all the existing fake data. 
	Apparently, this assumption is too strong and requires the generator to offer a high representation ability, since our true data manifold is always low-dimensional. However, we need only the existence of a generated fake data point such that its feature can be linearly represented by that of two true data points with at least one labeled data point. 
	This existence has a high probability to be satisfied given only a connected subdomain $\Omega^{jl}$ and an imperfect generator. 
	%\textbf{???This existence is very likely} given only a connected subdomain $\Omega^{jl}$ and an imperfect generator. 
	Here, we actually release the requirement for a complement generator. 
\end{enumerate}

\textbf{The Reasonableness of Assumption \ref{assump1} (3) on Synthetic Data.}
% We design the following synthetic data case to verify the reasonableness of Assumption \ref{assump1} (3). 
We design a binary classification problem with $\Omega = \Omega^{1} \cup \Omega^{2}$, where $\Omega^{1} = \Omega^{11} \cup \Omega^{12}$ and $\Omega^{11}$, $\Omega^{12}$ and $\Omega^{2}$ are three bounded 2D circles with no interaction. 
The unlabeled data are uniformly sampled from $\Omega$. 
If Assumption \ref{assump1} (3) is satisfied, each connected domain contains at least one labeled data point, as shown in Figure \ref{fig-3} (a1).
%we are lucky to have sampled data in both subdomains $\Omega_\mathcal{U}^{11}$ and $\Omega_\mathcal{U}^{12}$ of class 0, 
Then, as shown in Figure \ref{fig-3} (a4), we can easily train the discriminator to learn the correct decision boundary. 
However, if Assumption \ref{assump1} (3) is not satisfied, for the data in class $1$, we sample only the labeled data in subdomain $\Omega^{11}$, as shown in Figure \ref{fig-3} (b1). 
At this time, it is difficult for the discriminator to learn a correct decision boundary. 
As shown in Figure \ref{fig-3} (b4), (b6), and (b8), we choose three kinds of results under several training processes. For the unlabeled data in subdomain $\Omega^{12}$, the classification results are uncertain.
%they may all be predicted correctly, or they may all be wrong, or some are correct. 
Hence, a comparison between Figure \ref{fig-3} (a4), (b4), (b6), and (b8) can show that the condition of Assumption \ref{assump1} (3) is necessary to obtain a perfect discriminator on unlabeled data. 
Furthermore, as shown in Figure \ref{fig-3} (a3), (b3), (b5) and (b7), a comparison of the generated samples shows that there is indeed a certain gap between the generator distribution and the data distribution, while the optimal generator is not a complementary one, as claimed in \cite{Dai2017good}.
% Moreover, although the generator is unperfect for generating data, it is helpful for the discriminator to learn a correct boundary on unlabeled data.
\begin{figure*}[htb]
	\centering
	\includegraphics[scale=0.52]{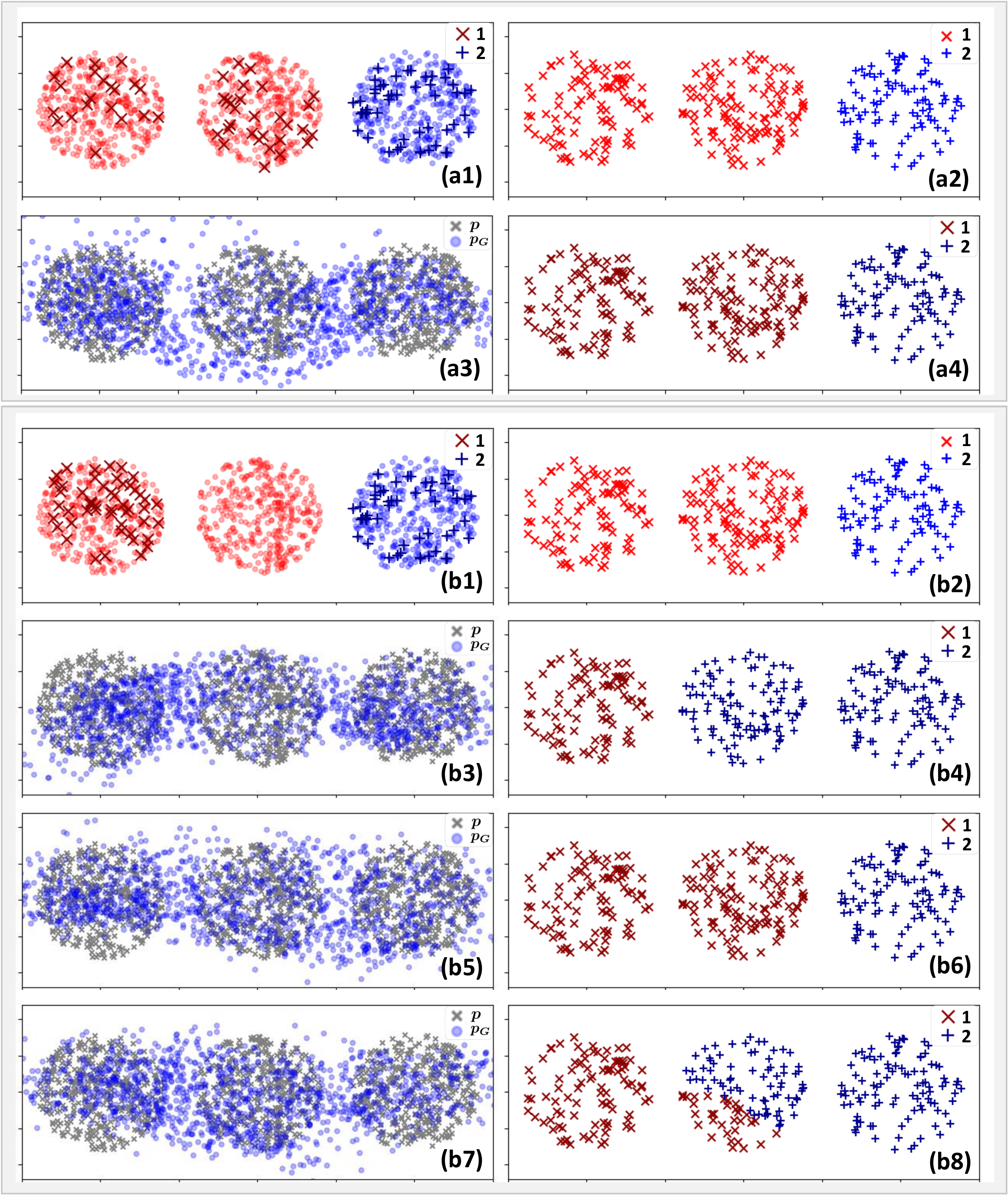}
	\vspace{0.1cm}
	\caption{The results of a binary classification problem on synthetic data when Assumption \ref{assump1} (3) is (a) or not (b) satisfied. 
		Labeled and unlabeled data are denoted by crosses and points, respectively, in (a1) and (b1). 
		Different colors indicate different classes in (a2), (a4), (b2), (b4), (b6), and (b8), where (a2) and (b2) are the ground truths on the test data and (a4), (b4), (b6), and (b8) are our results. 
		Generated and real data are denoted by blue points and gray crosses, respectively, in (a3), (b3), (b5), and (b7).
	}
	\label{fig-3}
\end{figure*}

\section{Conclusions}
Via a theoretical analysis, this paper answers the question of how GAN-SSL works. 
%First, given a fixed generator, the equivalence of optimizing the discriminator of GAN-SSL and optimizing the one of supervised learning is proved (Proposition \ref{prop1}). 
%That means the optimal discriminator in GAN-SSL is expected to be a perfect one on the labeled data.  
%Second, if the perfect discriminator can further touch its theoretical maximum, the optimal generator would be a perfect generator to approximate the true data distribution. 
%Since it's impossible to touch the theoretical maximum in practical, one can not expect to obtain a perfect generator for generating data (Proposition \ref{prop2}). 
%Furthermore, under the reasonable assumption that the labeled data can traverse all connected subdomains of the data manifold, we theoretically prove that the optimal discriminator in GAN-SSL is additionally expected to be a perfect one on the unlabeled data by learning an unperfect generator (Proposition \ref{prop3}).  
In conclusion, semi-supervised learning based on GANs will yield a perfect discriminator on both 
labeled (Proposition \ref{prop1}) and unlabeled data (Proposition \ref{prop3}) by learning an imperfect generator (Proposition \ref{prop2}), i.e., GAN-SSL can effectively improve the generalization ability in semi-supervised classification. In the future, the theoretical problems of more complex models, such as Triple-GAN and other methods, will be studied. In addition, the existence of Assumption \ref{assump1} (4) will undergo further theoretical and empirical investigations. 
%\ref{assump1} (4) also needs more theoretical and empirical investigations. 

\section*{Acknowledgements}
This work was supported in part by the Innovation Foundation of Qian Xuesen Laboratory of Space Technology, and in part by Beijing Nova Program of Science and Technology under Grant Z191100001119129. 

% ---- Bibliography ----
\bibliographystyle{ieee_fullname}
\bibliography{reference-File}

\clearpage
\section*{Appendix}
\subsection*{Proof of Proposition $1$}
\begin{proof}
Proof of Proposition $1$ (1).
Similar to the proof of Lemma $1$, given an optimal solution $D_{L} = (\omega, f)$ of the supervised objective $L_{D}$, 
due to the discriminator has infinite capacity, there exists $D^{*} = (\omega^{*}, f^{*})$ such that for all $x$ and $k\leq K$,
\beq\label{eq-1}
\exp(\omega_{k}^{*T}f^{*}(x)) 
= \frac{\exp(\omega_{k}^{T}f(x))}{\sum_{i=1}^{K}\exp(\omega_{i}^{T}f(x))}
\cdot \frac{p(x)}{p_{G}(x)}
\eeq
For all $x$,
\begin{align*}
&P_{D^{*}}(y|x, y\leq K)
= \frac{\exp(\omega_{y}^{*T}f^{*}(x))}{\sum_{i=1}^{K}\exp(\omega_{i}^{*T}f^{*}(x))} \\
=& \frac{\exp(\omega_{y}^{T}f(x))}{\sum_{i=1}^{K}\exp(\omega_{i}^{T}f(x))}
=P_{D_{L}}(y|x, y\leq K) 
\end{align*}
Then $L_{D^{*}} = L_{D_{L}}$. 
%The objective in Eq.~(\ref{eq-obj}) can be written as 
%\begin{equation}
%J_{D} = L_{D} +\mathbb{E}_{x\sim p}\log (1-P_{D}(K+1|x)) 
%+\mathbb{E}_{x\sim p_{G}}\log P_{D}(K+1|x).
%\end{equation}
Based on the definition and given Eq.~(\ref{eq-1}), we can obtain 
\begin{equation*}
P_{D^{*}}(K+1|x) = \frac{1}{1+\sum_{i=1}^{K}\exp(\omega_{i}^{*T}f^{*}(x))} = \frac{p_{G}}{p+p_{G}}.
\end{equation*}
By the proof of Lemma $1$, $D^{*}$ is an optimal solution of $U_{D}$. 
Because $D_{L}$ maximizes $L_{D}$, $D^{*}$ also maximizes $L_{D}$. It follows that $D^{*}$ maximizes $J_{D}$.

Proof of Proposition $1$ (2). First, we should note that if $D^{*}$ maximizes the GAN-SSL objective $J_{D}$, then $D^{*}$ maximizes $U_{D}$. Otherwise, based on Lemma $1$, there exists another solution $D^{'} = (\omega^{'}, f^{'})$ such that $ L_{D^{'}} = L_{D^{*}}$ and $U_{D^{'}} > U_{D^{*}}$, i.e., $J_{D^{'}} > J_{D^{*}}$, leading to contradiction. 
That is to say, for any optimal solution $D^{*} = (\omega^{*}, f^{*})$ of the GAN-SSL objective $J_{D}$, $U_{D^{*}}$ reaches the extreme value; 
then, $D^{*}$ is also an optimal solution of $L_{D}$. Otherwise, due to the infinite capacity of the discriminator, there exists an optimal solution $D_{1}$ of the supervised objective $L_{D}$. 
Thus, based on Proposition $1$ (1), there exists $D_{1}^{*} = (\omega_{1}^{*}, f_{1}^{*})$ such that $L_{D_{1}^{*}} = L_{D_{1}} > L_{D^{*}}$ and $U_{D_{1}^{*}} = U_{D^{*}}$. 
Therefore, $J_{D_{1}^{*}} > J_{D^{*}}$, leading to contradiction, i.e., for any optimal solution $D^{*} = (\omega^{*}, f^{*})$ of $J_{D}$, $D^{*}$ is an optimal solution of $L_{D}$. 
%then maximizing $J_{D}$ is equivalent to maximize the supervised objective $L_{D}$, i.e., optimizing the discriminator of GAN-SSL is equivalent to optimizing the one of supervised learning. 

Based on Proposition $1$ (1) and (2), we can obtain that maximizing $J_{D}$ is equivalent to maximizing the quantity $L_D$ and $U_D$, simultaneously. 
Then, the optimal solution of $J_D$ must also be the optimal solution of $U_D$. 
Similar to the theoretical results of [Goodfellow et al. 2014], for $G$ fixed, the optimal discriminator $D^*$ of the GAN-SSL objective is 
\begin{equation*}
P_{D^{*}}(K+1|x)= \frac{p_{G}(x)}{p(x)+p_{G}(x)},
\end{equation*}
and 
\begin{equation*}
P_{D^{*}}(y\leq K|x)= 1 - P_{D^{*}}(K+1|x) = \frac{p(x)}{p(x)+p_{G}(x)}.
\end{equation*}
This completes the proof.
\end{proof}

\subsection*{Proof of Proposition $2$}
\begin{proof}
By the definition of the discriminator $D$, $P_{D}(K+1|x) = \frac{1}{1+\sum_{i=1}^{K}\exp(\omega_{i}^{T}f(x))}$, 
and based on Proposition $1$ (3), $P_{D^*}(K+1|x) =\frac{p_{G}(x)}{p(x)+p_{G}(x)}$, 
then for the optimal discriminator $D = (\omega, f)$, $\frac{1}{1+\sum_{i=1}^{K}\exp(\omega_{i}^{T}f(x))} = \frac{p_{G}(x)}{p(x)+p_{G}(x)}$, such that
\begin{align*}
\sum_{i=1}^{K}\exp(\omega_{i}^{T}f(x))
&= \exp(\omega_{y}^{T}f(x)) 
+ \sum_{i\neq y}^{K}\exp(\omega_{i}^{T}f(x))\\
&= p(x)/p_{G}(x)
\end{align*} 
For a fixed optimal discriminator $D$, suppose there exists $0<\epsilon\ll 1$ such that the other logit output ($k\neq y$) of $D$ satisfies $(K-1)\exp (\omega_{k}^{T}f(x))\leq \epsilon$,
%Suppose the perfect discriminator $D$ is good enough, i.e., $P_{D}(y|x, y \leq K)\rightarrow 1$ and there exists $0<\epsilon\ll 1$ such that the other logits output of $D$ satisfies $(K-1)\exp (\omega_{k}^{T}f(x))\leq \epsilon$, 
then, 
%\bea 
%p/p_{G} - \exp(\omega_{y}^{T}f(x)) = \sum_{i\neq y}^{K}\exp(\omega_{i}^{T}f(x))\leq \epsilon
%\eea
$\exp(\omega_{y}^{T}f(x)) \geq p/p_{G}-\epsilon$. If the minimum can be achieved, i.e., $\exp(\omega_{y}^{T}f(x)) = p/p_{G}-\epsilon$,
therefore, 
\begin{align*}
L_{DG} &= \mathbb{E}_{(x,y)\sim p(x,y)}\log P_{D}(y|x, y\leq K)\\
&= \mathbb{E}_{(x,y)\sim p(x,y)}\log(\frac{\exp(\omega_{y}^{T}f(x))}{\sum_{i=1}^{K}\exp(\omega_{i}^{T}f(x))})\\
&= \mathbb{E}_{x\sim p(x)}\log((\frac{p(x)}{p_{G}(x)}-\epsilon)  \frac{p_{G}(x)}{p(x)})\\
&= -\int_{x}p(x)\log \frac{p(x)}{p(x)-\epsilon p_{G}(x)}dx
\end{align*}
Then, 
\bena
C(G) &= L_{DG} + U_{G}\\
&= -\int_{x}p(x)\log \frac{p(x)}{p(x)-\epsilon p_{G}(x)}dx\\
&+ \int_{x}p\log\frac{p(x)}{p(x)+p_{G}(x)} 
+ p_{G}\log \frac{p_{G}(x)}{p(x)+p_{G}(x)}dx\\
&= -KL(p||p-\epsilon p_{G})+2JS(p||p_{G}) - 2\log 2 
\eena 
This completes the proof.
\end{proof}

\subsection*{Proof of Proposition $3$}
%\begin{proposition}\label{prop3}
%	Given the conditions in Assumption \ref{assump1}, for all class $k \leq K$, for all data space points $x\in \Omega^{k}$, we have $\omega_{k}^{T}f(x) > \omega_{j}^{T}f(x)$ for any $j\neq k$.
%\end{proposition}
\begin{proof}  
First, if $x_k$ is a labeled data point, based on Assumption $1$ (1), we have $\omega_{k}^{T}f(x_k) > \omega_{j}^{T}f(x_k)$ for any $j\neq k$.
Then, we consider $x_k\in \Omega^{k}$ is an unlabeled data point. Without loss of generality, suppose $j = \text{arg}\max_{j\neq k}\omega_{j}^{T}f(x_k)$. Now, we prove it by contradiction. 

Suppose there exists a data space point $x_k\in \Omega^{k}$ and a class $j\neq k$, such that
\begin{equation}\label{eq-2}
\omega_{k}^{T}f(x_k) - \omega_{j}^{T}f(x_k)\leq 0
\end{equation}
By Assumption $1$ (3) and Assumption $1$ (4), there exists a connected subdomain $\Omega^{jl}$, a labeled data point $x_j\in \Omega^{jl}$, 
and a generated data point $x_g \in \Omega_{\mathcal{G}} \setminus \bar{\Omega}$, such that $f(x_g) = \alpha f(x_k) + (1 - \alpha) f(x_j)$
with $0 < \alpha < 1$.
Based on Assumption $1$ (2), $\omega_{j}^{T}f(x_g) < 0$. Thus, 
\begin{equation*}
\omega_{j}^{T}f(x_g) = \alpha \omega_{j}^{T}f(x_k) + (1 - \alpha) \omega_{j}^{T}f(x_j) < 0
\end{equation*}
By Assumption $1$ (1) and Assumption $1$ (2), for any $(x_j, j)\in \Omega_\mathcal{L}$, $\omega_{j}^{T}f(x_j)=\max_{k}\omega_{k}^{T}f(x_j)>0$. 
Moreover, by Eq. (\ref{eq-2}) and Assumption $1$ (2), for any $x_k\in \Omega^{k} \subset \Omega$, $\omega_{j}^{T}f(x_k) =\max_{i=1}^{K}\omega_{i}^{T}f(x_k) > 0$. 
Then, $\alpha \omega_{j}^{T}f(x_k) + (1 - \alpha) \omega_{j}^{T}f(x_j) > 0$, leading to contradiction. 
In summary, for all data space points $x_k\in \Omega^{k}$, we have $\omega_{k}^{T}f(x_k) > \omega_{j}^{T}f(x_k)$ for any $j\neq k$.
\end{proof}
\end{document}